\documentclass[11pt,a4paper]{article}
\usepackage[hyperref]{acl2021}
\usepackage{times}
\usepackage{latexsym}

\usepackage{microtype}
\usepackage{graphicx}
\usepackage{amsfonts}

\usepackage{amsmath}
\usepackage{amsthm}
\theoremstyle{definition}
\newtheorem{definition}{Definition}[section]
\newtheorem{proposition}{Proposition}

\newtheorem{condition}{Condition}

\aclfinalcopy % Uncomment this line for the final submission
\title{Attention Flows are Shapley Value Explanations}

\author{Kawin Ethayarajh \\
  Stanford University \\
  \texttt{kawin@stanford.edu} \\\And
  Dan Jurafsky \\
  Stanford University \\
  \texttt{jurafsky@stanford.edu} \\}

\begin{document}
\maketitle
\begin{abstract}
Shapley Values, a solution to the credit assignment problem in cooperative game theory, are a popular type of explanation in machine learning, having been used to explain the importance of features, embeddings, and even neurons. In NLP, however, leave-one-out and attention-based explanations still predominate. Can we draw a connection between these different methods? We formally prove that --- save for the degenerate case --- attention weights and leave-one-out values cannot be Shapley Values. \emph{Attention flow} is a post-processed variant of attention weights obtained by running the max-flow algorithm on the attention graph. Perhaps surprisingly, we prove that attention flows are indeed Shapley Values, at least at the layerwise level. Given the many desirable theoretical qualities of Shapley Values --- which has driven their adoption among the ML community --- we argue that NLP practitioners should, when possible, adopt attention flow explanations alongside more traditional ones.
\end{abstract}

\section{Introduction}

The approaches to model interpretability taken by the ML and NLP communities overlap in some areas and diverge in others. Notably, in machine learning, model prediction has sometimes been framed as a cooperative effort between the potential subjects of an explanation (e.g., input tokens) \citep{lundberg2017unified}. But how should we allocate the credit for a prediction, given that some subjects contribute more than others (e.g., the sentiment words in sentiment classification)? The Shapley Value is a solution to this problem that uniquely satisfies several criteria for equitable allocation \citep{shapley1953s}. However, while Shapley Value explanations have been widely adopted by the ML community --- to analyze the importance of features, neurons, and even training data \citep{ghorbani2019data,ghorbani2020neuron} --- they have had far less traction in NLP, where leave-one-out and attention-based explanations still predominate.

What is the connection between these different paradigms? When, if ever, are attention weights and leave-one-out values effectively Shapley Values? The adoption of Shapley Values --- which have their origins in game theory \citep{shapley1953s} --- by the ML community can be ascribed to their many desirable theoretical qualities. For example, consider a token whose masking out does not impact the model prediction in any way, regardless of how many other tokens in the sentence are also masked out. In game theory, such a token would be called a \emph{null player}, whose Shapley Value is guaranteed to be zero \citep{myerson1977graphs,young1985monotonic}. If we could provably identify the conditions under which attention weights and leave-one-out values are Shapley Values, we could extend such theoretical guarantees to them as well. 

In this work, we first prove that --- save for the degenerate case --- attention weights and leave-one-out values cannot be Shapley Values. More formally, there is no set of \emph{players} (i.e., possible subjects of an explanation, such as tokens) and \emph{payoff} (i.e., function defining prediction quality) such that the values induced by attention or leave-one-out also satisfy the definition of a Shapley Value. We then turn to \emph{attention flow}, a post-processed variant of attention weights obtained by running the max-flow algorithm on the attention graph \citep{abnar-zuidema-2020-quantifying}. We prove that when the players all come from the same layer (e.g., tokens in the input layer), there exists a payoff function such that attention flows are Shapley Values.

This means that under certain conditions, we can extend the theoretical guarantees associated with the Shapley Value to attention flow as well. As we show, these guarantees are axioms of faithful interpretation, and having them can increase confidence in interpretations of black-box NLP models. For this reason, we argue that whenever possible, NLP practitioners should use attention flow-based explanations alongside more traditional ones, such as gradients \citep{feng2018pathologies,smilkov2017smoothgrad}. We conclude by discussing some of the limitations in calculating Shapley Values for any arbitrary player set and payoff function in NLP.

\section{Model Interpretation as a Game}

The Shapley Value \citep{shapley1953s} was proposed as a solution to a classic problem in game theory: When a group of players work together to achieve a payoff, how can we fairly allocate the payoff to each player, given that some contribute more than others? The players here are the potential subjects of the explanation (e.g., input tokens); the payoff is some quality of the model prediction (e.g., correctness). We contextualize the game theoretic terms with respect to model interpretability below. 

\begin{definition}
A \emph{player} is a possible subject of the explanation (e.g., character, token, embedding, neuron). $N = \{1, ..., n\}$ is the set of all players. 
\end{definition}

\begin{definition}
A \emph{coalition} is a subset of players $S \subseteq N$ that work together. There are $2^n$ possible coalitions. The other players $N \setminus S$ are left out by being replaced with a non-subject that cannot affect the outcome (e.g., a zeroed-out embedding or a dropped-out neuron).
\end{definition}

\begin{definition}
The \emph{payoff} reflects some quality of the model prediction --- e.g., correctness, confidence, entropy --- made using a given coalition. It is defined by a \emph{payoff function} $v : 2^N \to \mathbb{R}$, where $v(\emptyset) = 0$. The \emph{value} $\phi_i(v)$ of a player $i$ is the share of the payoff allocated to it. In other words, it is the importance accorded to subject $i$ of an explanation.
\end{definition}

\begin{definition}
A \emph{game} is defined by $(N,v)$, a player set $N$ and payoff function $v$. It is a \emph{transferable utility} game (TU-game), where the payoff can be distributed among the players as desired. In the game of model interpretation, the subjects of the explanation are framed as players working cooperatively to make the best possible prediction.
\end{definition}

\subsection{Equitable Allocation}\label{ssec:criteria}

How can we allocate the payoff equitably, in a way that reflects the actual contribution made by each player? In other words, how can we faithfully interpret a prediction? The game theory literature proposes that any equitable payoff allocation satisfies these three conditions \citep{myerson1977graphs,young1985monotonic,ghorbani2019data}:

\begin{condition}{(\emph{Null Player})}: A player that induces no change in the payoff from joining any coalition has zero value. Formally, $\forall\ S \subseteq N \setminus \{i\}, v(S \cup \{i\}) = v(S) \implies \phi_i = 0$.
\label{condition:null_player}
\end{condition}

\begin{condition}{(\emph{Symmetry})}: Two players who induce the same change in payoff upon joining every coalition (that excludes them) have the same value. Formally, $\forall\ S \subseteq N \setminus \{i,j\}, v(S \cup \{i\}) = v(S \cup \{j\}) \implies \phi_i = \phi_j$.
\label{condition:symmetry}
\end{condition}

\begin{condition}{(\emph{Additivity})}: The value of a player across two different games with payoff $v,w$ should be the sum of its value in each game. Formally, $\forall\ i \in N, \phi_i(v + w) = \phi_i(v) + \phi_i(w)$.
\end{condition}

\subsection{The Shapley Value}

The Shapley Value is a well-known solution to the problem of payoff allocation in a cooperative setting, as it uniquely satisfies the three criteria for equitable allocation in \ref{ssec:criteria} \citep{shapley1953s,myerson1977graphs,young1985monotonic}. It sets the value of a player to be its expected incremental contribution to a coalition, over all possible coalitions.

\begin{definition}
Where $R$ is one of $n!$ possible permutations of the player set $N$, let $P_{R[:i]}$ be the subset of players that precede player $i$ in the permutation. Then, for a given payoff function $v$, the Shapley Value of player $i$ is \begin{equation}
    \phi_i(v) = \frac{1}{n!} \sum_{R} [ v(P_{R[:i]} \cup \{i\}) - v(P_{R[:i]}) ]
\label{definition:shapley}
\end{equation}
There are other equivalent ways of expressing the Shapley Value, including as a sum over the $2^n$ possible coalitions.
\end{definition}

In addition to satisfying our three criteria of equitable allocation (\ref{ssec:criteria}), a Shapley Value distribution always exists and is unique for a TU-game $(N, v)$. Unlike with attention weights, which have been criticized for allowing counterfactual explanations \citep{jain2019attention,serrano2019attention}, there can thus be no counterfactual Shapley Value distribution for a given input and payoff function $v$. The distribution is also said to be \emph{efficient}, since it allocates all of the payoff: $v(N) = \sum_{i \in N} \phi_i(v)$ \citep{myerson1977graphs,young1985monotonic}. The Shapley Value can, in theory, be computed for any player set and payoff function. However, in practice, there are typically too many players to calculate this combinatorial expression exactly. Generally, estimates are taken by uniformly sampling $m$ random permutations $\mathcal{R}$ \citep{ghorbani2019data}:
\begin{equation}
    \hat{\phi_i}(v) = \frac{1}{m} \sum_{R \in \mathcal{R}} [ v(P_{R[:i]} \cup \{i\}) - v(P_{R[:i]}) ]
\label{definition:shapley_approx}
\end{equation}
In the rest of this paper, we ask: Is there \emph{some} TU-game $(N, v)$ for which attention weights / attention flows / leave-one-out values are Shapley Values? If so, for which games?

\section{Attention Weights}
\label{sec:attention}

Many have argued that attention weights are not a faithful explanation, on the basis of \emph{consistency} (i.e., poor correlation with other importance measures) and \emph{non-exclusivity} (i.e., multiple explanations leading to the same outcome) \citep{jain2019attention}. Others have countered that they have some utility \citep{wiegreffe2019attention}. Without making assumptions about their inherent utility, we prove in this section that they cannot be Shapley Value explanations, outside of the degenerate case.

\begin{proposition}
If some player is attended to more than another, there is no TU-game $(N, v)$ for which attention weights are Shapley Values.
\label{prop:attention_bad}
\end{proposition}

\begin{proof}
Assume that attention weights are Shapley Values for some TU-game. Shapley Values are necessarily efficient (i.e., $v(N) = \sum_i \phi_i(v)$) \citep{myerson1977graphs,young1985monotonic}, so for attention weights to be efficient, the only applicable payoff function would be the sum of attention weights. Since each player only has one Shapley Value for a given $v$, if it is attended to multiple times, its value must be the \emph{total} attention paid to it: where $a_{j,i}$ denotes the attention $j$ pays to $i$, $\phi_i(v) = \sum_{j \in N} a_{j,i}$. Note that the payoff for a coalition $S$ is within some constant of its cardinality, since for a player $j$, the weights $a_{j,\cdot}$ of the players that it attends to sum to 1 \citep{bahdanau2014neural}. We consider two cases.

\paragraph{Case 1} For a player $j$ that attends to some other player, its contribution to the payoff of every $S \in N \setminus \{j\}$ is $\sum a_{j,\cdot} = 1$, implying $\phi_j(v) = 1$ by the Shapley Value definition (\ref{definition:shapley}). If some player (that pays attention) is more or less attended to than another --- which is the point of using attention --- this results in a contradiction. Thus $\phi_j$ cannot be the total attention paid to $j$. 

\paragraph{Case 2} For a player $i$ that doesn't attend to any other player, its contribution to the payoff of every $S \in N \setminus \{i\}$ is 0, since the attention paid to $i$ is redistributed among other players when it is absent. This implies $\phi_i(v) = 0$ by (\ref{definition:shapley}). However, all input embeddings fall under this case, and we know at least one will be attended to; its attention weights will be non-zero, making this a contradiction. Thus $\phi_i$ cannot be the total attention paid to $i$. 
\end{proof}

\section{Attention Flows}

\begin{figure*}
    \centering
    \includegraphics[width=0.9\textwidth]{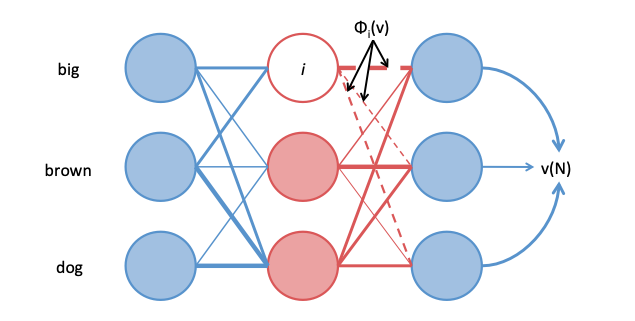}
    \caption{The attention flow network for three tokens across three layers, with player nodes (red) and non-player nodes (blue). The payoff $v(N)$ is the total flow through the network. $\phi_i(v)$ is the total outgoing flow of player $i$. Note that if we remove player $i$, then the total flow will decrease by $\phi_i(v)$, but the outgoing flow of the other two players (red) will stay the same. In other words, the contribution of player $i$ to the total flow $v(N)$ is always $\phi_i(v)$; therefore, $\phi_i(v)$ is its Shapley Value. This construction is possible because the players are all in the same layer and therefore parallel; if one depended on another, then its outgoing flow could not be its Shapley Value.}
    \label{fig:my_label}
\end{figure*}

What if we restricted the players to those from the same layer of a model? The remaining players still affect the prediction but can't have any of the payoff allocated to them. In this case, attention weights still cannot be Shapley Values. However, attention weights can be post-processed. \citet{abnar-zuidema-2020-quantifying} proposed treating the self-attention graph as a flow network --- where the attention weights are capacities --- and then applying a max-flow algorithm \citep{ford1956maximal} to this network to calculate the maximum flow on each edge. We prove (by construction) that these \emph{attention flows} are Shapley Values when the players are restricted to those from the same layer and the payoff is the total flow, as visualized in Figure \ref{fig:my_label}.

\begin{proposition}
Consider a TU-game $(N, v)$, where $N = \{1,...,n\}$ players are all from the same layer. Let $f$ denote the flow obtained by running a max-flow algorithm on the graph defined by the self-attention matrix, where the capacities are the attention weights. Let $v(S) = |f(S)|$, the \emph{value of the flow} when only permitting flow through players in the coalition $S \subseteq N$. Then for each player $i$, its total outflow $|f_o(i)|$ is its Shapley Value.
\label{proposition:attention_flow}
\end{proposition}

\begin{proof}
Blocking the flow through a player $i \in S$ decreases $v(S)$ by that player's outflow $|f_o(i)|$, since the attention flow is only calculated once --- with the entire graph --- and not for each possible subgraph. Since the players are all disjoint and have no connections, blocking the flow through one player does not affect the outflow of any of the other players. This would not be the case, for example, if the players were in different layers, in which case changes in flow upstream would cause changes in flow downstream. Then for any coalition $S \subseteq N$ and player $i \not\in S$, $v(S \cup \{i\}) = v(S) + |f_o(i)|$. We can rewrite the total outflow for player $i$ as
$$\begin{aligned} |f_o(i)| &= v(S \cup \{i\}) - v(S), \forall\ S \subseteq N\\ 
&= \frac{n!}{n!}  v(S \cup \{i\}) - v(S), \forall\ S \subseteq N \\
&= \frac{1}{n!} \sum_{R} \left[ v(P_{R[:i]} \cup \{i\}) - v(P_{R[:i]}) \right] \end{aligned}$$
which is just the Shapley Value definition (\ref{definition:shapley}). Note that the players cannot be from different layers --- at least for the definition of $v$ as the total flow value --- because the Shapley Value distribution would not be efficient (i.e., $v(N) \not= \sum_{i \in N} \phi_i(v)$) and efficiency necessarily holds for Shapley Values. This in turn implies that the theoretical properties that hold for Shapley Values extend to attention flows under these conditions.
\end{proof}

\paragraph{Attention Rollout} \citet{abnar-zuidema-2020-quantifying} also proposed another post-processed variant of attention called \emph{attention rollout}, in which the attention weight matrices from each layer are multiplied with those before it to get aggregated attention values. Attention roll-out values cannot be Shapley Values, however; this can be shown with a trivial extension of the proof to Proposition \ref{prop:attention_bad}.

\section{Leave-One-Out}

\emph{Erasure} describes a class of interpretability methods that aim to understand the importance of a representation, token, or neuron by erasing it and recording the resulting effect on model prediction \cite{li2016understanding,arras2017relevant,feng2018pathologies,serrano2019attention}. Although the Shapley Value technically falls under this class, most erasure-based methods only remove one entity --- the one whose importance they want to estimate ---  and this only takes two forward passes, compared to $O(2^n)$ passes for the Shapley Value. Since only one entity is erased, this simpler group of erasure-based methods is called \emph{leave-one-out} \citep{jain2019attention,abnar-zuidema-2020-quantifying}. We show in this section that leave-one-out values are not Shapley Values, except in the degenerate case.

\begin{proposition}
If $\exists\ i \in N$ such that player $i$ is not a null player even when excluding the coalition $N \setminus \{i\}$, then there is no TU-game $(N, v)$ for which leave-one-out values are Shapley Values.
\end{proposition}

\begin{proof}
Let the leave-one-out value of player $i$ be denoted by $\text{LOO}_i(v)$. Let $R'$ denote any permutation of $N$ where $P_{R'[:i]} \not= N \setminus \{i\}$. By definition,
\begin{equation*}
    \begin{split}
        \phi_i(v) &= \frac{1}{n!} \sum_{R} \left[ v(P_{R[:i]} \cup \{i\}) - v(P_{R[:i]}) \right] \\
        &= \frac{1}{n!} \sum_{R'} \left[ v(P_{R'[:i]} \cup \{i\}) - v(P_{R'[:i]}) \right] \\
        & \qquad + \frac{1}{n} \underbrace{\left(v(N) - v(N \setminus \{i\})\right)}_{\text{LOO}_i(v)} \\
    \end{split}
\end{equation*}
By our assumption, the first term is non-zero, so there is no equivalence with $\text{LOO}_i(v)$. In practice, this assumption is almost always satisfied.
\end{proof}

Note that leave-one-out tells us very little about player importance for discrete payoff functions. For example, if the payoff were the correctness (i.e., 1 if correct and 0 otherwise), then the importance of a player would be binary: it would either be critically important to prediction or totally irrelevant. This provides an incomplete picture --- while there is enough redundancy in BERT-based models to tolerate some missing embeddings, this does not mean those embeddings are of no importance \citep{kovaleva2019revealing,ethayarajh2019contextual,michel2019sixteen}. For example, if two representations played a critical and identical role in a prediction --- but only one was necessary --- then leave-one-out would assign each a value of zero, despite both being important. In contrast, the Shapley Value of both players would be non-zero and identical.

\section{Applications}

Because Shapley Values have many useful applications, attentions flows --- and any other score that meets the criteria for a Shapley Value  --- have many useful applications as well:

\begin{itemize}
    \item For one, using the various properties of the Shapley Value, we can provide more specific interpretations of model behavior than is currently the case, backed by theoretical guarantees. For example, if a token has zero attention flow in layer $k$ but non-zero flow in layer $k - 1$, then we can conclude that all the information it contains about the label (e.g., sentiment) was extracted by the model prior to the $k$th layer; this derives from the ``null player'' property of the Shapley Value. The same could not be said if the token only had a leave-one-out value of zero, since leave-one-out values are not Shapley Values.
    
    \item Interpretability in NLP often takes a single token or embedding to be the unit of analysis (i.e., a ``player'' in game theoretic terms). However, what if we wanted to understand the role of entire groups of tokens rather than individual ones? For most interpretability methods, there is no canonical way to aggregate scores across multiple units --- we cannot necessarily add the raw attention scores of two tokens, since the usefulness of one may depend on the other. If we used a method that provided Shapley Values, we could easily redefine a ``player'' to be a group of tokens, such that all tokens in the same player group would simultaneously be included or excluded from a coalition.
    
    \item Recent work has used the Data Shapley --- an extension of the Shapley Value --- to estimate the contribution of each example in the training data to a model's decision boundary \citep{ghorbani2019data}. If we're fine-tuning BERT for sentiment classification, for example, we might want to know which sentence is more helpful: ``This movie was great!'' or ``This was better than I expected.'' We can answer such questions by using the Data Shapley. To our knowledge, this has been done in computer vision but not in NLP.
\end{itemize}

\section{Limitations and Future Work}

Because Shapley Values --- and by extension, attention flows --- have many theoretical guarantees that are axioms of faithful interpretation, we encourage NLP practitioners to provide attention flow-based explanations alongside more traditional ones. This is not without limitations, however. As proven in Proposition \ref{proposition:attention_flow}, this equivalence only holds for a specific payoff function --- the total flow through a layer --- which is reflective of model confidence but not of the prediction correctness.

But why do we need attention flows at all if, in theory, Shapley Values can be calculated for any arbitrary player set and payoff function? While this is true in theory, because of the combinatorial calculation (\ref{definition:shapley}), it is computationally intractable in most cases. While it is possible to take a Monte Carlo estimate (\ref{definition:shapley_approx}), in practice the bounds can be quite loose \citep{maleki2013bounding}. Finding TU-games for which the Shapley Value can be calculated exactly in polynomial time –- as with attention flow -– is an important line of future work. These explanations may come with trade-offs: for example, SHAP is a kind of Shapley Value that assumes contributions are linear (i.e., a coalition can't be greater than the sum of its parts), which makes it much faster to calculate but restricts the set of possible payoff functions \citep{lundberg2017unified}. Still, such methods will be critical to providing explanations that are both fast and faithful.

\section*{Acknowledgements}

We thank Rishi Bommasani and the reviewers for their helpful feedback. KE was supported by an NSERC PGS-D and the Stanford Institute for Human-Centered AI.

\bibliographystyle{acl_natbib}
\bibliography{anthology,acl2021}

% \appendix

% \section{Example Appendix}
% \label{sec:appendix}

% This is an appendix.

\end{document}